\documentclass[nohyperref]{article}

\newif\ifarXiv
\arXivtrue

\PassOptionsToPackage{sort}{natbib}

\usepackage{microtype}
\usepackage{graphicx}
\usepackage{subfigure}
\usepackage{booktabs} 

\usepackage{hyperref}

\usepackage{paralist}
\usepackage{svg}
\usepackage{graphicx,import}

\ifarXiv
  \usepackage{twoml}
\else
  \usepackage{icml2022_TAGML}
\fi

\usepackage{amsmath}
\usepackage{amssymb}
\usepackage{mathtools}
\usepackage{amsthm}

\usepackage{bbm}

\usepackage[capitalize,noabbrev]{cleveref}

\theoremstyle{plain}
\newtheorem{theorem}{Theorem}[section]

\theoremstyle{definition}
\newtheorem{definition}[theorem]{Definition}

\theoremstyle{remark}

\newtheorem{example}[theorem]{Example}

\usepackage[textsize=tiny]{todonotes}

\icmltitlerunning{Metric Space Magnitude and Generalisation in Neural Networks}

\graphicspath{{../}{icml2022_tagml}}

\begin{document}

\twocolumn[
\icmltitle{Metric Space Magnitude and Generalisation in Neural Networks\\
            }

\icmlsetsymbol{equal}{*}
\icmlsetsymbol{equal_last}{$\dagger$}

\begin{icmlauthorlist}
\icmlauthor{Rayna Andreeva}{yyy}
\icmlauthor{Katharina Limbeck}{Helmholtz,TUM}
\icmlauthor{Bastian Rieck}{equal_last,Helmholtz,TUM}
\icmlauthor{Rik Sarkar}{equal_last,yyy}
\end{icmlauthorlist}

\icmlaffiliation{yyy}{School of Informatics, University of Edinburgh}
\icmlaffiliation{Helmholtz}{Helmholtz Munich}
\icmlaffiliation{TUM}{Technical University of Munich}

\icmlcorrespondingauthor{Rayna Andreeva}{r.andreeva@sms.ed.ac.uk}

\vskip 0.3in
]

\printAffiliationsAndNotice{}

\begin{abstract}
Deep learning models have seen significant successes in numerous applications, but their inner workings remain elusive. The purpose of this work is to quantify the learning process of deep neural networks through the lens of a novel topological invariant called  \emph{magnitude}. Magnitude is an isometry invariant; its properties are an active area of research as it encodes many known invariants of a metric space. We use magnitude to study the internal representations of neural networks and propose a new method for determining their generalisation capabilities. Moreover, we theoretically connect magnitude dimension and the generalisation error, and demonstrate experimentally that the proposed framework can be a good indicator of the latter.
\end{abstract}

\section{Introduction} 

\begin{figure*}
    \centering
    \includegraphics[width=\textwidth, trim={0 1.3cm 0 0}]{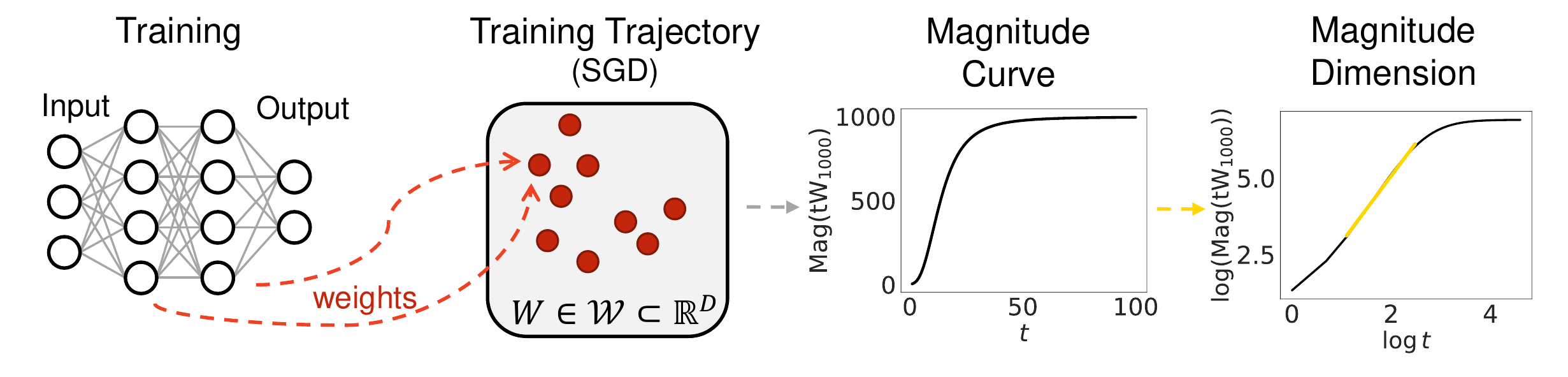}
    \caption{\textbf{Overview of the procedure.} We first train a neural network and monitor the training trajectory. At each 1000 iterations of the trajectory, we collect the training weights $W \in \mathbb{R}^d$ into a point cloud. Then, we compute the magnitude curve of the point cloud at selected scales $t$, create the log-log plot and estimate the magnitude dimension based on it.}
    \label{fig:overview}
\end{figure*}

Deep neural networks (DNNs) have become ubiquitous due to their remarkable performance in a range of tasks, including computer vision \cite{xie2017aggregated}, natural language processing \cite{vaswani2017attention}, and scientific discovery~\cite{jumper2021highly}. However, state-of-the-art DNNs often comprise millions to billions of parameters, making it impractical for human users to gain a precise understanding of the inner workings of these networks. One crucial yet unanswered question is to understand the generalisation of neural networks, i.e. their ability to perform well on unseen data. In recent literature, various notions of neural networks' intrinsic dimensions have been proposed \cite{simsekli2020hausdorff,birdal2021intrinsic,dupuis2023generalization}, which demonstrate that the most important ingredient for generalisation is effective capacity rather than the number of parameters.

In this work, we propose a novel measure of the intrinsic dimension of neural networks, based on a novel topological invariant called magnitude. Unlike previous approaches, magnitude benefits from a simple interpretation, and potentially better computational complexity. This choice is also theoretically justified. Magnitude is an isometry invariant of a metric space and its properties are an active area of mathematical research due to the fact that it encodes many important invariants from geometric measure theory and integral geometry \cite{leinster2013magnitude}. Further, magnitude has roots in theoretical ecology and it can capture the effective number of species in an ecosystem. In a broader context, it has been shown that magnitude can thus encode the effective number of distinct points in a space \cite{leinster2013magnitude}. We further build on this idea and propose a novel method for evaluating the generalisation error using the concept of an effective number of models, which acts as an effective capacity. 

Although there has been significant theoretical research on magnitude, a considerable number of its characteristics are yet to be discovered, and several unanswered questions persist, particularly regarding its practical applications in machine learning. While recent studies \citep{bunch2021weighting, adamer2021magnitude} have started to establish links between magnitude vectors and machine learning applications, this area of research is still in its very early stages.

Recently, it has been shown that a concept derived from magnitude, known as the magnitude dimension, is the same as the Minkowski dimension \cite{meckes2015magnitude} under certain conditions. As a result, we can theoretically show that the generalisation error of the trajectories of a training algorithm is intrinsically linked to the magnitude dimension of the so-defined metric space. This enables us to generate novel insights about the learning process of neural networks. Further, our contribution is the first work exploring the magnitude dimension, which has solid theoretical foundations, as a notion of the intrinsic dimension of neural networks.

\paragraph{Contributions.}
Our work is the first to introduce the mathematical concept of magnitude into theoretical deep learning and the study of neural network generalisation;
we believe that this is an exciting novel contribution to the nascent field of topological machine learning~\citep{Hajij22, hensel2021survey}. After establishing a theoretical bound, we explore the change in the magnitude function across multiple experimental settings, and compare the value of magnitude at multiple scales to the test accuracy. Further, we demonstrate experimentally that there is a link between magnitude and the test accuracy. By making a novel connection with the persistent homology dimension, we then compare our proposed measure with the intrinsic dimension introduced by \citet{adams2020fractal} and demonstrate that ours benefits from a better computational complexity and interpretability. 
%
%
In short, our contributions are as follows:
\begin{compactitem}
    \item We propose a novel method for evaluating the generalisation of neural networks based on magnitude and the effective number of models, which allows us to monitor  performance without a validation set.
    \item We prove a new upper bound for the generalisation error, linking the generalisation error to a magnitude-based characteristic of the training trajectories.
    \item We empirically show that the evolution of these measures throughout the training process correlates with the accuracy of the test set.
    \item We prove that all notions of previously proposed intrinsic dimensions are the same as the magnitude dimension, and we verify this result empirically.
\end{compactitem}

\section{Related Work}

We briefly review the literature related to generalisation in neural networks, intrinsic dimension, and magnitude.

\paragraph{Generalisation Bounds and Intrinsic Dimension.} 
Several works explore intrinsic dimension for capturing the generalisation capabilities of neural networks.
\citet{simsekli2020hausdorff} demonstrated that the fractal dimension of a hypothesis class is associated with the generalisation error, which is further linked to the heavy-tailed behavior of the trajectory of networks \cite{simsekli2019tail,hodgkinson2021multiplicative,mahoney2019traditional}. However, many assumptions were required for the bound to be computed in practice. A more recent work relaxed some of the assumptions, and developed the notion of the persistent homology dimension \cite{adams2020fractal}, $\mathrm{dim_{PH}}$.
The authors in \citet{birdal2021intrinsic} were the first to offer a theoretical justification for using topological invariants for the analysis of deep neural networks. Another work \cite{magai2022topology} investigated $\mathrm{dim_{PH}}$ at different depths and layers of the network and observed its evolution. In \citet{dupuis2023generalization}, the authors developed a data-driven dimension and compared it with the $\mathrm{dim_{PH}}$ of \citet{birdal2021intrinsic}. They have demonstrated stronger correlation with the generalisation error than previously shown and managed to relax some of the restrictive assumptions.

\paragraph{Magnitude and its Applications in Machine Learning.} 
Magnitude was first proposed in \citet{solow1994measuring} for  measuring biodiversity, albeit without any reference to its mathematical properties. It was only approximately twenty years later when \citet{leinster2013magnitude} formalised its mathematical properties using the language of category theory. Further, magnitude has been realised as the Euler characteristic in magnitude homology \cite{leinster2021magnitude}.
%
While magnitude has theoretical foundations, its applications to machine learning are scarce. Recently, there has been renewed interest in introducing magnitude into the machine learning community. The first work to develop the concept of magnitude in the context of machine learning  demonstrated that the individual summands of magnitude, known as magnitude weights, can be used as an efficient boundary detector \cite{bunch2021weighting}. Further, it has been used for working in the space of images and it has demonstrated its usefulness as an effective edge detector \cite{adamer2021magnitude}. However, our contribution constitutes the first direct application of magnitude to deep learning.

\paragraph{Using Topology to Characterise Neural Networks.}
Earlier research has established a connection between neural network training and topological invariants~\citep{fernandez2021determining}, using topological complexity as proxy for generalisation performance, for instance~\citep{Rieck19a}. However, these studies focused solely on analyzing the trained network after completing the training process \cite{fernandez2021determining}, potentially missing critical aspects of the training dynamics \cite{birdal2021intrinsic}. 
By contrast, we propose the use of another topological invariant---magnitude---which affords more interpretability than previous approaches. Moreover, we compute magnitude on the training trajectories instead of on the trained network, offering crucial topological insights into training dynamics.

\section{Background}

\begin{figure}[tbp]
    \centering
    \centerline{\includegraphics[width=0.5\textwidth]{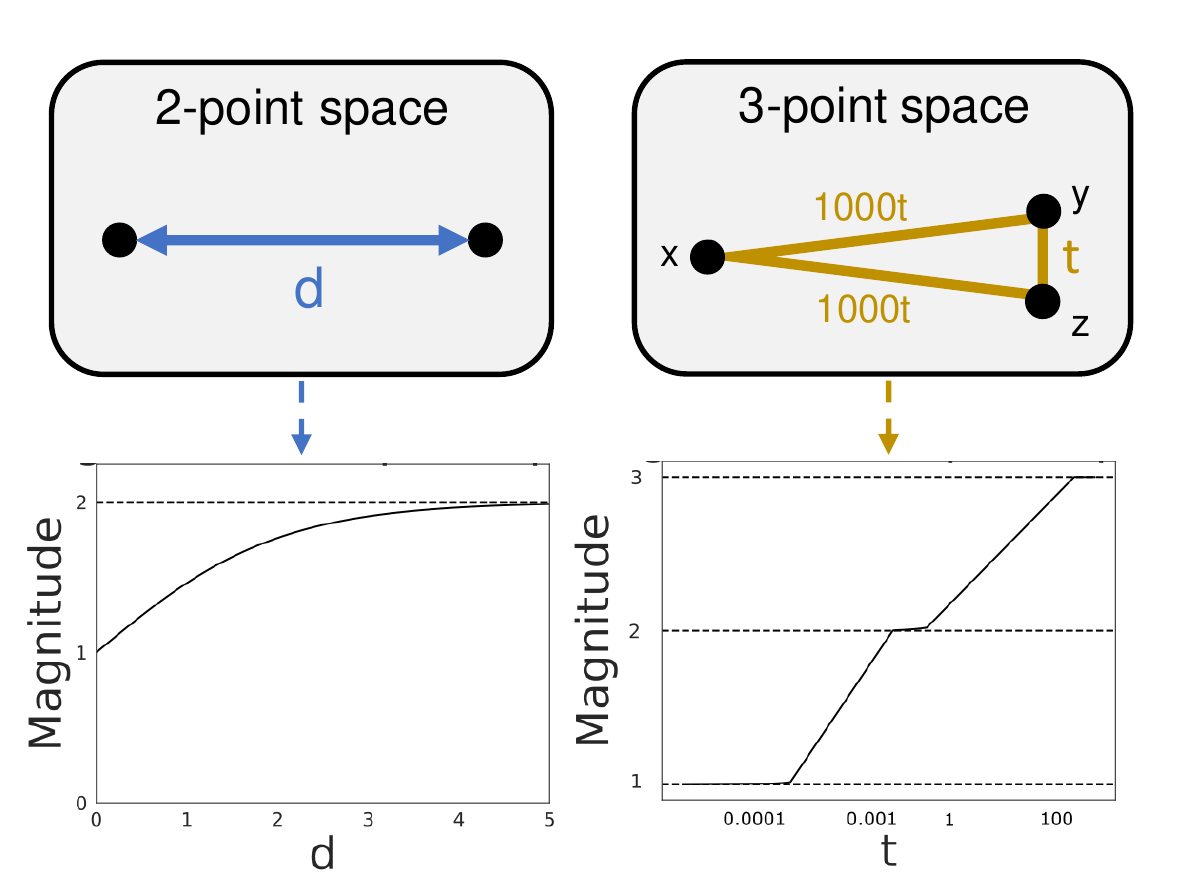}}

    \caption{ \textbf{Magnitude of the 2- and 3-point space.} On the left, we see the 2-point space, where the distance between the points is $d$. On the right we see the 3-point space for an isosceles triangle with distance $t$ between $y$ and $z$, and $1000t$ between $x$ and $y$ and $x$ and $z$. Below each space, we see the respective magnitude function.}
    \label{fig:magnitude_function_visualisation}
\end{figure}

We start by defining magnitude and the magnitude dimension, and then proceed to define intrinsic dimensions.

\subsection{Magnitude}

While the magnitude of metric spaces is a general concept~\cite{leinster2013magnitude}, we restrict our focus to subsets of $\mathbb{R}^n$, where magnitude is known to exist~\cite{meckes2013positive}.
\begin{definition}
    Let $(X,d)$ be a finite metric space with distance metric $d$. Then, the similarity matrix of $X$ is defined as $\zeta_{ij} = e^{-d_{ij}}$ for $1 \leq i,j \leq n$, where $n$ denotes the cardinality of $X$.
\end{definition}

\begin{definition}
Let $X$ be a metric space with similarity matrix $\zeta_{ij}$. 
If $\zeta_{ij}$ is invertible, magnitude is defined as
\begin{equation}
\mathrm{Mag}(X) = \sum_{ij}(\zeta^{-1})_{ij}.
\end{equation}

\label{magnitude}
\end{definition}

When $X$ is a finite subset of $\mathbb{R}^n$, then $\zeta_{ij}$ is a symmetric positive definite matrix as proven in \citet{leinster2013magnitude} (Theorem 2.5.3). Then, $(\zeta^{-1})_{ij}$ exists, and hence magnitude exists as well.
Magnitude is best illustrated when considering a few sample spaces with a small number of points.
\begin{example}
    Let $X$ denote the metric space with a single point $a$. Then, $\zeta_X$ is a $1 \times 1$ matrix with $\zeta_X^{-1} = 1$ and using the formula for magnitude, we get $\mathrm{Mag}_X = 1$.
\end{example}

\begin{example}
A more illustrative example is given by the space of two points. Let $X=\{a,b\}$ be a finite metric space where $d_X(a,b)=d$. Then
\begin{equation}
   \zeta_X=\begin{bmatrix} 1 & e^{-d}\\ e^{-d} & 1\end{bmatrix},  
\end{equation}

so that
\begin{equation}
    \zeta_X^{-1} = \frac{1}{1 - e^{-2d}} \begin{bmatrix} 1 & -e^{-d}\\ -e^{-d} & 1\end{bmatrix},
\end{equation}
and therefore 
\begin{equation}
    \mathrm{Mag}(X)=\dfrac{2 - 2e^{-d}}{1 - e^{-2d}}=\dfrac{2}{1+e^{-d}}. \label{eq:1}
\end{equation}
\end{example}

This example is also illustrated in Figure \ref{fig:magnitude_function_visualisation}. Using Eq.~\eqref{eq:1}, if $d$ is very small, i.e. $d \to 0$, $\mathrm{Mag}(X) \to 1$. Similarly, when $d \to \infty$, $\mathrm{Mag}(X) \to 2$. In practice, as it can be seen from Figure \ref{fig:magnitude_function_visualisation}, $\mathrm{Mag}(X) = 2$ for a value of $d$ as small as $5$.

More information about a metric space can be obtained by looking at its rescaled counterparts. The resulting representation is richer, and it can be summarised compactly in the magnitude function. For this purpose, the scale parameter $t$ is introduced. By varying $t$, we obtain a scaled version of the metric space, which permits us to answer what the number of \emph{effective points} of a metric space is.

\paragraph{The Magnitude Function.}
Magnitude assigns to each metric space not just a scalar number, but a function. This works as follows: we introduce a scale parameter $t$, and for each value of $t$, we consider the space where the distances between points are scaled by $t$. More concretely, for $t>0$, the notation $tX$ means $X$ scaled by a factor of $t$. Computing the magnitude for each value of $t$ then gives the magnitude function.
More formally, we have the following definition:
\begin{definition}
Let $(X,d)$ be a finite metric space. We define $(tX, d_t)$ to be the metric space with the same points as $X$ and the metric $d_t(x,y) = td(x,y)$.
\end{definition}
The intuition here is that we are looking at the same space but through different scales, for example, changing the distances from metres to centimeters. Then the definition of the magnitude function follows naturally.
\begin{definition}
The magnitude function of a finite metric space $(X,d)$ is the partially-defined function $t \rightarrow \mathrm{Mag}(tX)$, which is defined for all $t \in (0, \infty)$.
\end{definition}

\begin{example}
    Consider the magnitude function of the 3-point space in Figure \ref{fig:magnitude_function_visualisation}. In this example, the points $x$, $y$ and $z$ form an isosceles triangle. When $t=0.0001$, all the three points are very close to each other and almost indistinguishable. Hence, we say that the space has 1 effective point. In contrast, when $t=0.01$, the distance between the two points on the right $y$ and $z$ is very small, and $x$ is quite far. Therefore, we say that the space looks like two points. Finally, when $t$ is large, all the three points are far away from each other, and the value of magnitude is 3, which is also the cardinality of the space.
\end{example}

\subsection{Intrinsic Dimension}

There are various notions that can be used to measure the intrinsic dimension of a space. In this work, we will focus on three such notions: the upper-box dimension (Minkowski), the magnitude dimension and the persistent homology dimension. The box dimension is based on
covering numbers and can be linked to generalization via \citet{simsekli2020hausdorff}, whereas the magnitude dimension is built upon the concepts we defined earlier.

\begin{definition}[Minkowski dimension]

For a bounded metric space $X$, let $N_{\delta}(X)$ denote the maximal number of disjoint closed $\delta$-balls with centers in $X$. The upper box/Minkowski dimension is defined as

\begin{equation}
\mathrm{dim}_{\mathrm{Mink}}X = \lim_{\delta \to 0} \sup \frac{\log(N_{\delta}(X))}{\log(\frac{1}{\delta})}.    
\end{equation}

\label{minkowski}
\end{definition}

There is a subtle point to be made here. In general, the Minkowski and Hausdorff dimensions do not coincide and are not equivalent. However, in \citet{simsekli2020hausdorff} the authors provide conditions under which the Hausdorff dimension of the space we are interested in coincides with the
Minkowski dimension. In fact, for many fractal-like sets, these two notions of dimensions are equal to each other; see \citet[Chapter~5]{mattila1999geometry}.

\begin{definition}[Magnitude dimension]
When
\begin{equation}
\label{mag_dimension}
\mathrm{dim_{Mag}}X = \lim_{t \to \infty}\frac{\log(\mathrm{Mag}(tX)))}{\log t}   
\end{equation}
exists, we define this to be the magnitude dimension of $X$~\citep{meckes2015magnitude}.
\label{magnitude_dimension}
\end{definition}

The magnitude dimension can be approximately interpreted as the rate of change of the magnitude function for a suitable interval of values for $t$.
We can introduce another notion of the fractal dimension, known as the persistent homology dimension ($\mathrm{dim_{PH}}$) \cite{adams2020fractal}.

\begin{definition}
The persistent homology dimension of a bounded metric space $(X,d)$, denoted by $\mathrm{dim_{PH}}$, is defined as
\begin{equation}
\inf\{\alpha > 0, \exists C > 0, \forall W \subset X \mathrm{finite}, E_{\alpha} < C\},
\end{equation}
where $E_{\alpha}$ is the $\alpha$-lifetime sum, defined as $$E_{\alpha}(W) \coloneqq \sum_{(b,d) \in PH^{0}(Rips(W))}(d-b)^{\alpha}, $$ and $b$ and $d$ are the birth and death values respectively for the persistent homology of degree 0 ($PH^{0}$). It measures all connected components in the Vietoris-Rips filtration of W, denoted by $Rips(W)$.
\label{persistent_homology_dimension}    
\end{definition}
The definition is rather technical and it is not crucial for the work here, for more details please refer to \citet{adams2020fractal,memoli2019primer,edelsbrunner2022computational}.

\section{Theoretical Results}

We first elucidate connections between different notions of intrinsic dimension before proving connections to the generalisation error.

\subsection{Connection between Notions of Intrinsic Dimension}

After we have introduced three different notions of intrinsic dimensions, we demonstrate that they are in fact the same under some mild assumptions. Our novel contributions is proving that $\mathrm{dim_{Mag}}X$ and $\mathrm{dim_{PH}^{0}}X$. The result is formalised in the following theorem, which assumes that all notions of dimension exist.

\begin{theorem}
Let $X \subset \mathbb{R}^n$ be a compact set and either $\mathrm{dim_{Mag}}X$ or $\mathrm{dim_{Mink}}X$ exist. Then 
\begin{equation}
    \mathrm{dim_{Mag}}X = \mathrm{dim_{PH}^{0}}X
\end{equation}
\label{magnitude_equals_ph}
\end{theorem}
\begin{proof}
Since $X$ is compact and either $\mathrm{dim_{Mag}}X$ or $\mathrm{dim_{Mink}}X$ exist, from Corollary 7.4, \cite{meckes2015magnitude} it follows that both $\mathrm{dim_{Mag}}X$ and $\mathrm{dim_{Mink}}X$ exist and $\mathrm{dim_{Mag}}X = \mathrm{dim_{Mink}}X$. Since $X$ is compact, from the Heine-Borel Theorem, $X$ is both closed and bounded, and therefore from a result in \citet{kozma2006minimal, schweinhart2021persistent}, we have that $\mathrm{dim_{Mink}}X = \mathrm{dim_{PH}^{0}}X$. Hence, $\mathrm{dim_{Mag}}X = \mathrm{dim_{Mink}}X = \mathrm{dim_{PH}^{0}}X$, which implies that $\mathrm{dim_{Mag}}X = \mathrm{dim_{PH}^{0}}X$.

\end{proof}

\subsection{Connection to the Generalisation Error}

After having established equality between the various notions of dimensions, we proceed to formalise the required language of machine learning theory, culminating in a novel generalisation result.

For the beginning of this section, we follow the notation from \citet{shalev2014understanding}. We briefly recall some standard definitions to make our paper self-contained.
In a standard statistical learning setting, $\mathcal{X}$ denotes the set of features and $\mathcal{Y}$ the set of labels. Together, the cross product $\mathcal{X} \times \mathcal{Y}$ represents the space of data $\mathcal{Z}$. The learner has access to a sequence of data of $m$ samples, called the training data, denoted by $S = ((x_1, y_1), ..., (x_m,y_m))$ in $\mathcal{X} \times \mathcal{Y} = \mathcal{Z}$. We assume that the training set $S$ is generated by some unknown probability distribution over $\mathcal{X}$, which we denote by $\mathcal{D}$, and that each of the samples $\{x_i, y_i\}$ are independent and identically distributed (i.i.d) samples from $\mathcal{D}$.
We will focus on a restricted search space for finding a set of predictors, which we call a hypothesis class, $\mathcal{H}$. Each element, called a hypothesis, $h \in \mathcal{H}$, is a function from $\mathcal{X}$ to $\mathcal{Y}$. An optimal hypothesis, or parameter vector $h \in \mathcal{H}$ is selected by computing a quantity called a loss function, defined by $\ell :\mathcal{H} \times \mathcal{Z} \rightarrow \mathbb{R}_{+}$. The empirical error is then $L_{S}(h) = \frac{1}{m}\sum_{i=1}^{m}\ell (h,z)$ for $z \in \mathcal{Z}$ and the true error or risk is $L_{\mathcal{D}}(h) = \mathbb{E}_{z}[\ell(h,z)]$ for $z \in \mathcal{Z}$. The generalisation error is defined as the difference between the true error and the empirical risk, or $|L_{S}(h) - L_{\mathcal{D}}(h)|$.

In the case of neural networks, the hypothesis class is $\mathcal{W} \subset \mathbb{R}^d$, and each $w \in \mathcal{W} $ is a parameter vector. Given a training algorithm $\mathcal{A}$, we want to study the set of all hypothesis classes returned from the optimisation procedure $\mathcal{A}$ for a given training data $S$. We call this the \textit{optimisation trajectories}, and denote them by $\mathcal{W}$. To access the iterates at every step of the process, we denote by $w_i$ an element of $\mathcal{W}$ at iteration $i$. More concretely, $\mathcal{W}:=\{w \in \mathbb{R}^d: \exists i \in [0,I], w = [\mathcal{A}(S)_i]\}$, where $I$ is the number of training iterations. In other words, when we fix $i$, we are interested in the weights at iteration $i$, returned by the optimisation algorithm $\mathcal{A}$.

For the following result, there is a more technical condition required from \citet[Assumption H1]{birdal2021intrinsic}, which can be found in the Appendix. We denote this assumption as H1. We require the existence of the constant $M > 0$, which quantifies how dependent the set $\mathcal{W}_S$ is on the training sample $S$.
Now that we have all the required definitions, we can proceed with the novel result. 

\begin{theorem}
\label{intrinsic_dimension_main_result}
    Let $\mathcal{W} \in \mathbb{R}^{d}$ be a compact set. Under the assumption that $H1$ holds, $\ell$ is bounded by a constant $C$ and $K$-Lipschitz continuous in $w$. For $n$ sufficiently large, we then have the following bound:
    \begin{equation}
    \begin{aligned}
    \sup_{w \in \mathcal{W}}|L_{S}(w) - L_{\mathcal{D}}(w)| \leq
    \\ 2C \sqrt{\frac{[\mathrm{dim_{Mag}}\mathcal{W} + 1]\log^2(nK^2)}{n} + \frac{\log(7M/\gamma)}{n}}
    \end{aligned}
    \end{equation}
    with probability $1-\gamma$ over $S \sim \mathcal{D}^{\bigotimes n}$, where $M$ is the constant from Assumption H1.
\end{theorem}

\begin{proof}
    Since $\mathcal{W}$ is bounded, we have $\mathrm{dim_{PH}^{0}}\mathcal{W} = \mathrm{dim_{Mag}}\mathcal{W}$. Therefore, the result follows from substituting $\mathrm{dim_{PH}^{0}}\mathcal{W}$ with $\mathrm{dim_{Mag}}\mathcal{W}$ in Proposition 1~\citep{birdal2021intrinsic}.
\end{proof}

\begin{figure*}[ht!]
    \centering
    \subfigure[]{\includegraphics[width=0.25\textwidth]{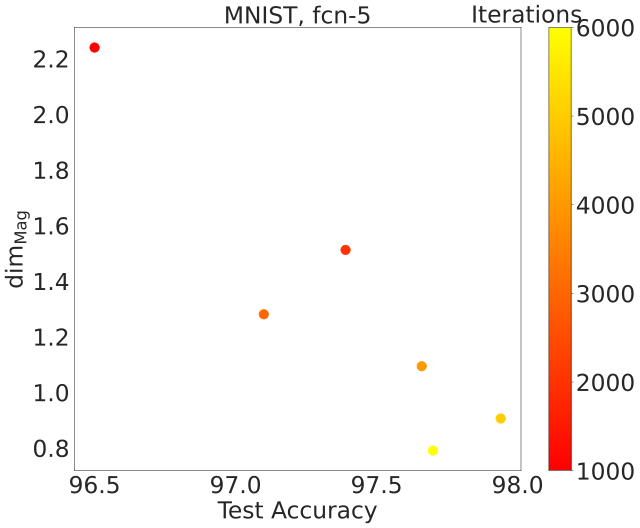}} 
    \subfigure[]{\includegraphics[width=0.25\textwidth]{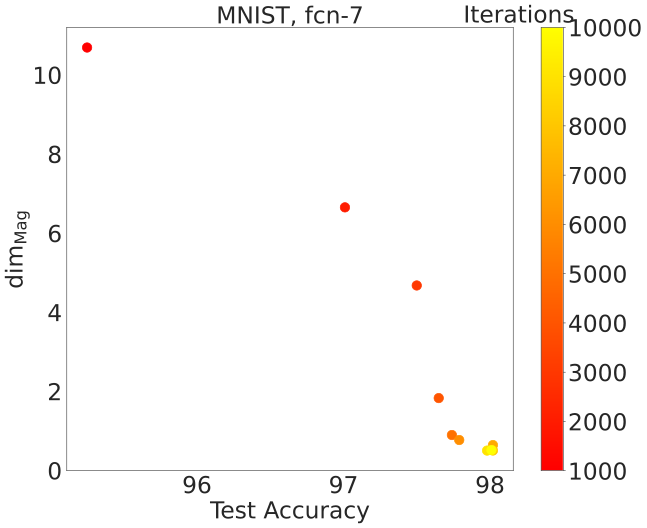}} 
    \subfigure[]{\includegraphics[width=0.25\textwidth]{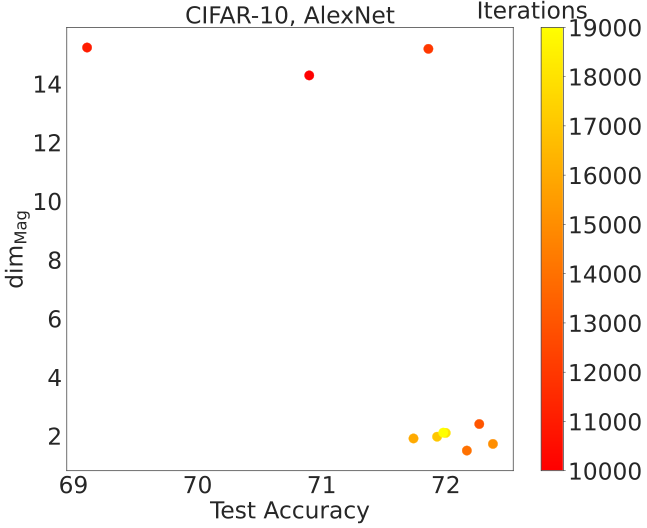}} 
    \caption{\textbf{The magnitude dimension correlates with the test accuracy.}
    In plots (a-c), we see the magnitude dimension plotted against the test accuracy over a varying number of iterations, which are depicted in different colours, from red to yellow. We note that there is correlation between the magnitude dimension and the test accuracy across both MNIST and CIFAR-10, and across different architectures (FCN-5, FCN-7, AlexNet), which is stronger for MNIST than for CIFAR-10.  }
    \label{fig:magnitude_dimension_various_plots}
\end{figure*}

\section{Methods}

In contrast with previous work on the intrinsic dimension, we propose to estimate the fractal dimension using the magnitude. As we have seen, the concept of magnitude dimension coincides with the Minkowski dimension and the persistent homology dimension. This theoretical connection enables us to confidently explore the concept of magnitude in the context of neural networks. We do this as follows: at selected points on the weight trajectory $\mathcal{W}$, we subsample a number of models $W$, which can be interpreted as a point cloud, where each point is a model in the model space. This space has a very high dimension equal to the number of parameters in the network. We compute the magnitude and the magnitude dimension of each such point cloud. We then investigate the connection between these quantities and the test accuracy.

In light of our novel result in Theorem \ref{intrinsic_dimension_main_result}, linking the intrinsic dimension and the generalisation bound, it is natural to ask if the magnitude function can also be used to explore the learning process of neural networks. Since the magnitude dimension can be roughly interpreted as the rate of change of the magnitude function, and therefore, if there is a link between the magnitude dimension and the generalisation error, then there should also be a link between the values of the magnitude function at different scales and the generalisation error. What we want to study is the change of the space of model trajectories as the learning process advances. In other words, does the space look like a bigger number of distinct points or does it resemble fewer points as the training progresses? Translating this idea to the language of magnitude, is the \textit{effective number of models} increasing or decreasing when performing more iterations of the learning algorithm?


Since magnitude is a function, we would like to take a cross-sectional slice of the magnitude curve and examine the magnitude values for a particular choice of the scale parameter $t$. Therefore, we formulate the definition of the effective number of models for a fixed $t$.

\begin{definition}
    We define the \textbf{effective number of models} as the value of $\mathrm{Mag}(t_i\mathcal{W})$ at scale $t_i$.
\end{definition}

\subsection{Analyzing Deep Neural Network Dynamics via the Magnitude Dimension}


Estimating the magnitude dimension is not a straightforward task, as the limit from Equation \ref{mag_dimension} needs to be approximated by finding the longest straight part of the magnitude function, which cannot be computed automatically, but needs to be done manually. Here we provide the algorithm for computation of $\mathrm{dim_{Mag}}\mathcal{W}$ from a finite sample $W$, approximating an infinite process. We follow the procedure similar to the estimations of the magnitude dimension \cite{willerton2009heuristic,o2023alpha}. The details are described in \cref{alg:example}. After choosing a suitable representative interval $[t_i, t_j]$, the log-log plot of magnitude versus $t$ is generated, and the slope $m$ of the line and the intercept $b$ are computed at the selected interval $[t_i, t_j]$. Then, $m$ is taken to be the estimate of $\mathrm{dim_{Mag}}\mathcal{W}$. This procedure can be essentially seen as computing the limit based on the slope, which is an application of l'Hôpital's rule.

\begin{algorithm}[tb]
   \caption{Estimation of $\mathrm{dim_{Mag}}\mathcal{W}$}
   \label{alg:example}
\begin{algorithmic}
   \STATE {\bfseries Input:} The set of the training trajectories $\mathcal{W} = \{w_{i}\}^n$ of size $n$, scale parameter $t:[0,t_k]$, interval $[t_i,t_j], i<j<k$
   \STATE {\bfseries Output:} $\mathrm{dim_{Mag}}$
   \FOR{$r=1$ {\bfseries to} $k$}
   \STATE $\mathrm{Mag}(t_r\mathcal{W}) \leftarrow \mathcal{W} $ 
   \ENDFOR
   \STATE $m, b \leftarrow$ fitline$(log(\mathrm{Mag}(\mathcal{W}_n)[t_i:t_j]), log([t_i:t_j]))$ 
   \STATE $\mathrm{dim_{Mag}}\mathcal{W} \leftarrow m$ 
\end{algorithmic}
\end{algorithm}

\begin{figure*}[ht!]
    \centering
    \subfigure[]{\includegraphics[width=0.24\textwidth]{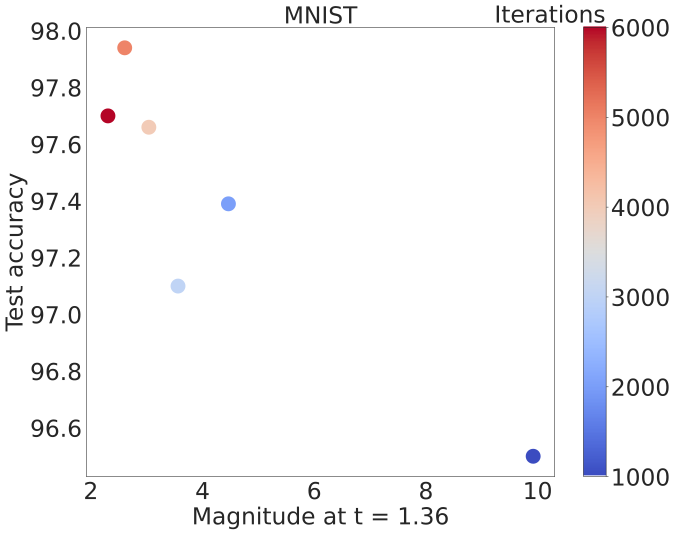}} 
    \subfigure[]{\includegraphics[width=0.24\textwidth]{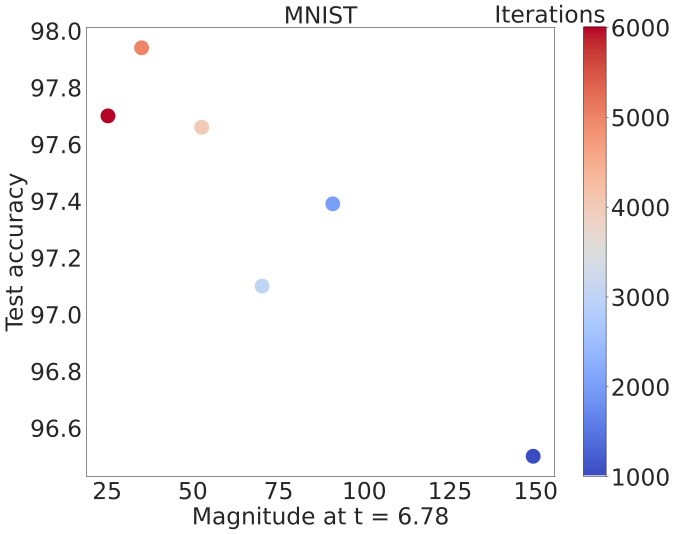}} 
    \subfigure[]{\includegraphics[width=0.24\textwidth]{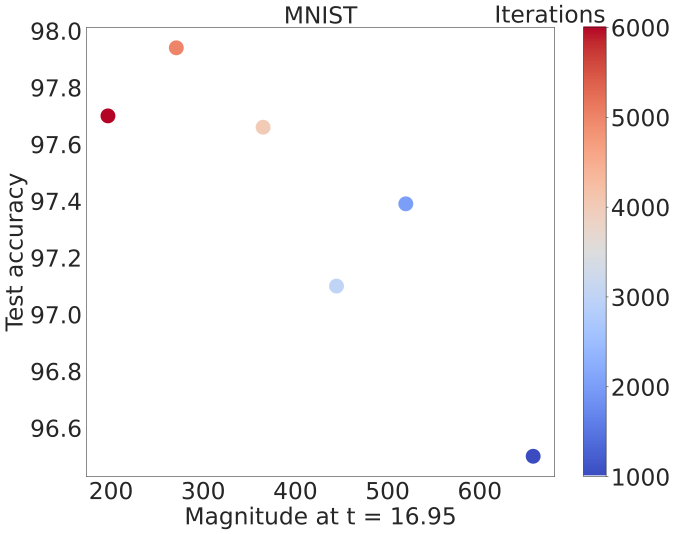}}
    \subfigure[]{\includegraphics[width=0.24\textwidth]{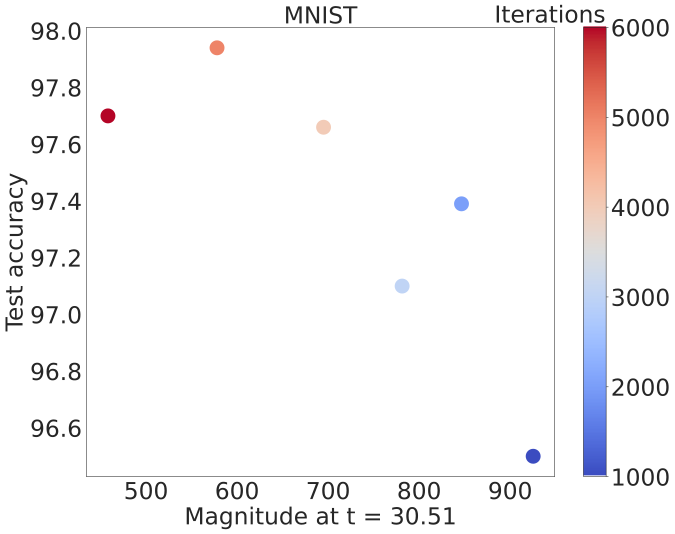}}
    \subfigure[]{\includegraphics[width=0.24\textwidth]{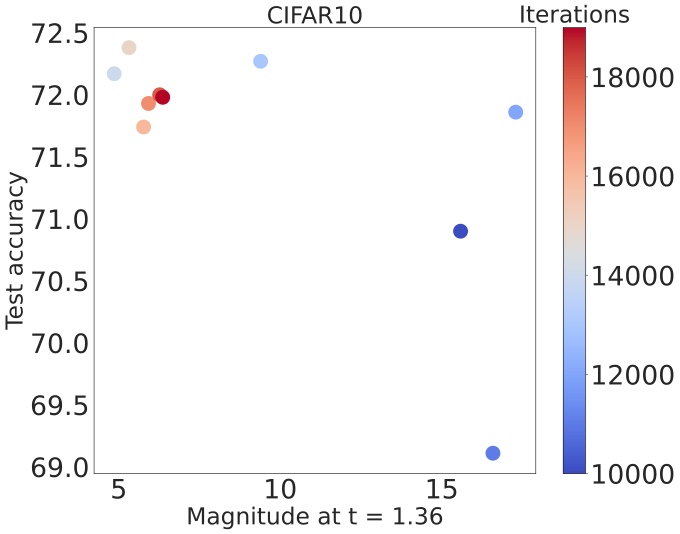}}
    \subfigure[]{\includegraphics[width=0.24\textwidth]{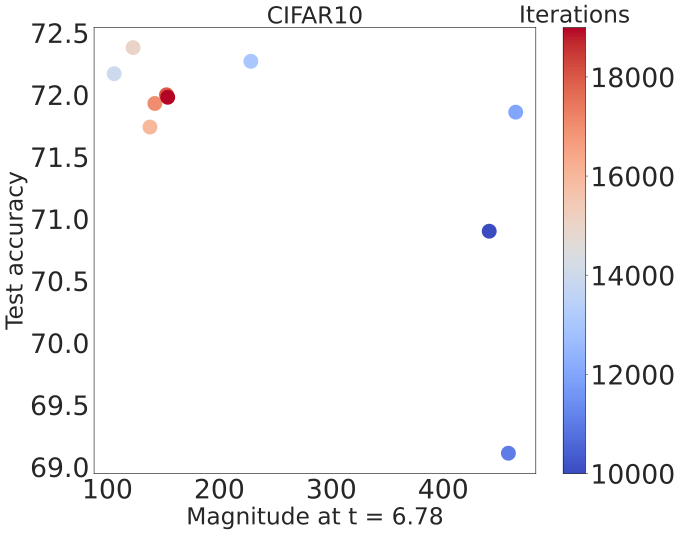}}
    \subfigure[]{\includegraphics[width=0.24\textwidth]{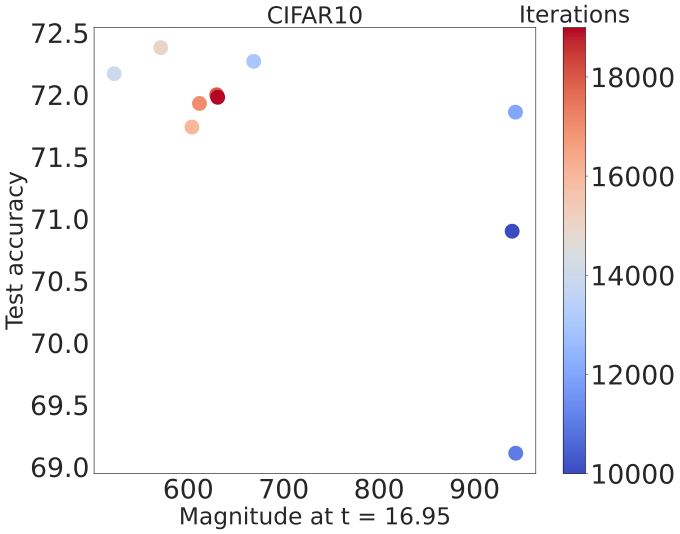}}
    \subfigure[]{\includegraphics[width=0.24\textwidth]{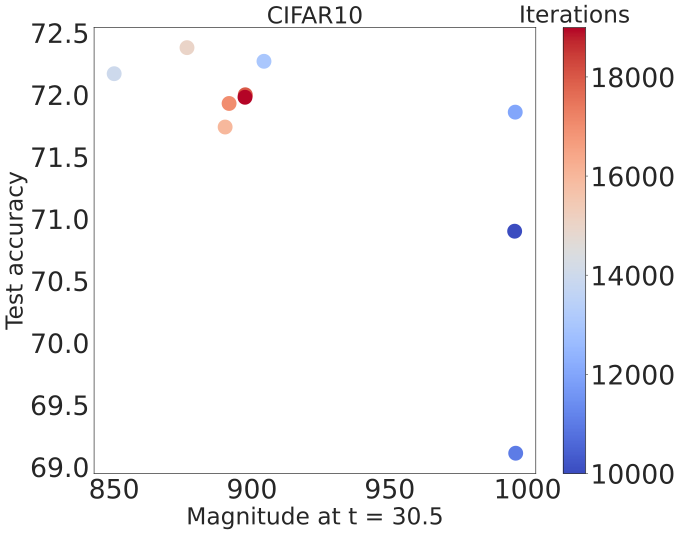}}
    \caption{\textbf{The effective number of models correlates with the test accuracy.} Here we see the cross-sectional evaluation of the magnitude curve at different values of $t$. Each point represents the model trajectory over 1000 iterations, over which the magnitude is computed, as well as the test accuracy at the last model in the sliding window. The first row (plots (a-d)) shows the magnitude for MNIST. The second row (plots (e-h)) shows the plots for CIFAR10. We note that across all scales, there is similar pattern of correlation between the test accuracy and the magnitude values. }
    \label{fig:magnitude_scales_learning_rates}
\end{figure*}

\section{Experimental Results}


The first goal of this section is to verify our  main claim, namely that our estimate of the magnitude dimension is capable of measuring generalisation. The second goal is to establish a connection between magnitude itself and the test accuracy. The third goal is to empirically compare the two very close measures of the fractal dimension, namely $\mathrm{dim_{Mag}}\mathcal{W}$ and $\mathrm{dim_{PH}}\mathcal{W}$. The fourth goal is to attempt to explain what our results mean for generalisation in neural networks in general, arriving at novel insights. In order to show this, we use our procedure on a number of different neural network architectures, training settings and datasets. In particular, we train a 5-layer (fcn-5) and 7-layer (fcn-7) on MNIST and AlexNet \cite{krizhevsky2017imagenet} on CIFAR10, over a different range of learning rates and batch size of 100. We consider a sliding window of 1000 training iterations. We then estimate $\mathrm{dim_{Mag}}\mathcal{W}$ of each window, following the steps in \cref{alg:example}.
 
\subsection{Exploring the learning process}

We assess the main claim of the paper, which links the magnitude dimension with the generalisation error. Figure \ref{fig:magnitude_dimension_various_plots} reveals multiple findings. First, we observe that there is a correlation between the magnitude dimension and the test accuracy---the lower the magnitude dimension, the higher the test accuracy. This result agrees with what has previously been observed in \citet{birdal2021intrinsic,simsekli2020hausdorff}, albeit from the perspective of \emph{magnitude}, an invariant that is arguably simpler to compute and more interpretable in practice than persistent homology. Second, this holds across both MNIST and CIFAR10, and also across different architectures, which shows that even though the parameters differ considerably, the intrinsic dimension of different datasets can still be similar. To achieve good generalisation performance, it is important to keep the dimension as small as possible, without losing important representational features by collapsing them onto the same dimension.

\begin{figure*}[ht!]
    \centering
    \subfigure[]{\includegraphics[width=0.26\textwidth]{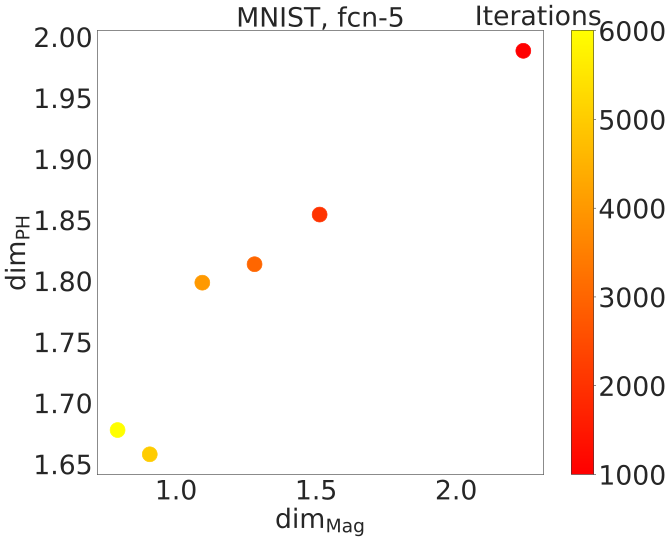}} 
    \subfigure[]{\includegraphics[width=0.22\textwidth]{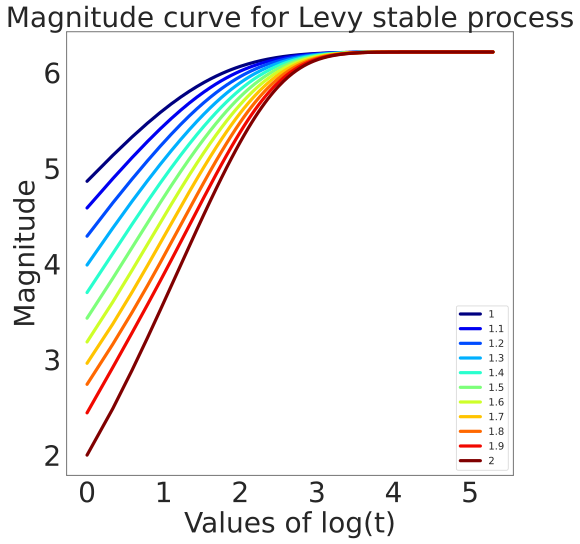}} 
    \subfigure[]{\includegraphics[width=0.22\textwidth]{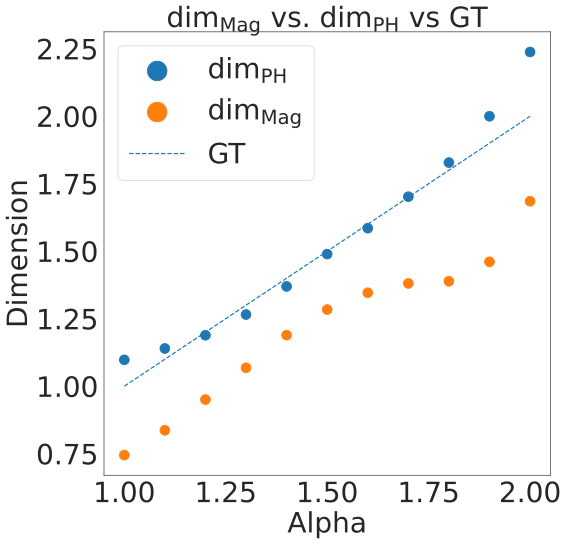}} 
    \caption{\textbf{The magnitude dimension, persistent homology dimension and ground truth for an $\alpha$-Levy stable process.} In plot (a) we compare the $\mathrm{dim_{Mag}}\mathcal{W}$ and $\mathrm{dim_{PH}}\mathcal{W}$ against the number of iterations. There is strong correlation of $0.96$, with statistically significant p-value ($p < 0.05$). In plot (b) we see the magnitude curves for different values of $\alpha$. In plot (c) we see the magnitude curves of the $\alpha$-stable Levy process for multiple values of $\alpha$.
     }
    \label{fig:ablation_study_various_plots}
\end{figure*}

\subsection{Analysing and visualising network trajectories}
Next, we want to investigate the effect on the magnitude function as the network is trained for more iterations. For this purpose, we select four values of $t$ across the range $[0, 40]$, $t \in \{1.36, 6.78, 16.95, 30.51\}$ and we compute the effective number of models $\mathrm{Mag}(t\mathcal{W})$ to give us a glimpse into the space of training trajectories. We further fix the learning rate and vary the number of iterations. The experiments were performed with a learning rate of $0.1$ for illustrative purpose. However, we note that this pattern holds for multiple learning rates. In Figure \ref{fig:magnitude_scales_learning_rates} we can see the resulting magnitude value at each of the selected scales of $t$, and the colour depicts the number of iterations. We note that there is a concentration of red points in the upper left corner, indicating that higher test accuracy corresponds to a lower value of magnitude. Similarly, for the blue points, the higher value of magnitude is linked to worse test set performance. This pattern holds across both MNIST and CIFAR10. Moreover, the lower the test accuracy, the higher the magnitude value, implying that models with lower magnitude values generalise better. This result intuitively makes sense. When the magnitude value is small, the space looks like a smaller number of points. With the increase of $t$, the magnitude converges to the cardinality of the set. The effective number of models in (a) equals approximately 2 for high test accuracy, which is an interesting phenomenon. This means that out of the 1000 models, for $t = 1.38$, the space looks like 2 models, indicating that there are 2 large clusters formed by the training trajectories. Surprisingly, this is not the case for network trained across 1000 iterations, with the lowest test accuracy. In that case, the space of models looks like 10 points and is more scattered.

Note that the model with high test accuracy has smaller magnitude even at larger scales, suggesting that the models exhibit some sort of clustering behavior. In the magnitude terminology, the effective number of distinct models is smaller when the network generalises better; a "good optimisation trajectory" has a lower magnitude than a "bad optimisation trajectory"; the learning algorithm is implicitly optimising the magnitude, by clustering the models into a small number of clusters. Throughout our experiments we thus observed strong correlation between magnitude itself and the test accuracy, which persists across the different scales, learning rates and datasets. Therefore, the main insight from the results in Figure \ref{fig:magnitude_scales_learning_rates} is that a small effective number of models is good for generalisation and it indicates that there is a clustering pattern.

\subsection{Similarities between $\mathrm{dim_{Mag}}\mathcal{W}$ and $\mathrm{dim_{PH}}\mathcal{W}$}

In Figure \ref{fig:ablation_study_various_plots}(a), we demonstrate that there is a strong correlation between $\mathrm{dim_{Mag}}\mathcal{W}$ and $\mathrm{dim_{PH}}\mathcal{W}$ on MNIST. Although there is a good correspondence between the two, $\mathrm{dim_{PH}}\mathcal{W}$ is more difficult to interpret. In order to give some interpretation, the authors had to take a step back and produce the persistent diagrams which were used to calculate the specified dimension. However, due to the fact that they have to sample the space to estimate the dimension, these calculations involve a big number of different diagrams, hence linking $\mathrm{dim_{PH}}\mathcal{W}$ with the original trajectory is not so straightforward. On the other hand, there is a more clear connection between magnitude and the magnitude dimension and the appearance of clusters of weight parameters.

\paragraph{Ablation study.}
In Figure \ref{fig:ablation_study_various_plots} we see the results from an ablation study. We simulate data from a process similar to the weight trajectories with known ground truth, using a $d$-dimensional $\alpha$-stable Levy  process \cite{seshadri1982fractal}, where we take $d=10$, and compare the values of the magnitude dimension to the fractal dimension. Further, we compare it with the persistent homology dimension, which as we have proven in Theorem \ref{magnitude_equals_ph}, is the same as the magnitude dimension. The purpose of this study is to empirically evaluate this theoretical result. As seen in Figure \ref{fig:ablation_study_various_plots}(b), our dimension highly correlates with the ground truth. It seems that it is consistently lower than the true dimension by approximately $0.3$. Nonetheless, the high statistically significant correlation indicates that the magnitude dimension is indeed very close to the fractal dimension for a process which resembles the iterations of SGD. This gives us high confidence that the magnitude dimension measures exactly what it is supposed to measure.

\section{Conclusion}

Our work provides the first connection between magnitude and the generalisation error. We proved a theoretical result linking the magnitude dimension of the optimisation trajectories and the test accuracy. We verified our results experimentally by exploring the evolution of magnitude across multiple experimental settings and datasets. We have further expanded our understanding about the clustering property of the weight trajectory: through both theoretical and empirical results, we demonstrated that models with better generalisation error tend to cluster more, compared to models with worse test performance, which are of higher magnitude and more spread out. This phenomenon has been previously described~\cite{bruel2019topology,birdal2021intrinsic}, and in this work our observations supplement it. In future work, we will improve on the estimation of the magnitude dimension and investigate the use of magnitude as a regulariser. 
Moreover, by using magnitude itself, we demonstrated that we can monitor the performance of the neural network without a validation set. In future work we can explore magnitude as an early stopping criteria, similar to \citet{Rieck19a}. Furthermore, from our empirical ablation study, the magnitude dimension seems to be consistently lower than the ground truth, which will be investigated further. In particular, we want to consider to what extent it is possible to improve on the generalisation bound in Theorem \ref{intrinsic_dimension_main_result}.

\subsection*{Acknowledgements}

The authors wish to thank Alexandros Keros and Emily Roff for helpful
comments and insightful discussions that helped improve the manuscript.
RA is supported by the United Kingdom Research and Innovation (grant
EP/S02431X/1), UKRI Centre for Doctoral Training in Biomedical AI at the
University of Edinburgh, School of Informatics and the International
Helmholtz-Edinburgh Research School for Epigenetics (EpiCrossBorders).
KL is supported by Helmholtz Association under the joint research school
``Munich School for Data Science - MUDS.''

\bibliography{main}
\bibliographystyle{icml2022}

\newpage
\appendix
\onecolumn
\section{Intrinsic dimensions and equalities}

The following result has been proven by \citet{meckes2015magnitude}, which related the magnitude dimension with the Minkowski dimension of compact subsets of Euclidean space. 

\begin{theorem}[Corollary 7.4, \cite{meckes2015magnitude}]
If $X \subset \mathbb{R}^n$ is compact and if $\mathrm{dim_{Mag}}X$ or $\mathrm{dim_{Mink}}X$ exist, then both exist and $\mathrm{dim_{Mag}}X = \mathrm{dim_{Mink}}X$.
\label{minkowski_equals_magnitude}
\end{theorem}

Using the theorem above, one can use the magnitude dimension to approximate the Minkowski dimension.

\begin{theorem}[Heine-Borel] Let $A \subset \mathbb{R}^n$. Then $A$ is compact if and only if $A$ is both closed and bounded.
\label{heine-borel}
\end{theorem}



\begin{theorem}\cite{schweinhart2021persistent}
Let $X \subset \mathbb{R}^n$ be a bounded set. Then $\mathrm{dim_{PH}^{0}}X = \mathrm{dim_{Mink}}X$.
\label{ph_equals_box}
\end{theorem}

\section{Assumption H1}
The following technical assumption has been introduced in previous work \cite{birdal2021intrinsic,simsekli2020hausdorff} and it is necessary for the statement of the novel bound in Theorem \ref{intrinsic_dimension_main_result}
For any $\delta > 0$, define the fixed grid on $\mathbb{R}^d$ to be

\begin{equation}
    G = \biggl\{ \left( \frac{(2j_i+1)\delta}{2\sqrt{\delta}},\ldots, \frac{(2j_d+1)\delta}{2\sqrt{\delta}}  \right): j_i \in \mathbb{Z}, i = 1, \ldots, d \biggl\}.
\end{equation}
The collection of centres of each ball is denoted by the set $N_{\delta}$. Now, we define 
$N_{\delta}(S) = \{ x \in N_{\delta}: B_d(x, \delta) \cap \mathcal{W_S} \neq \emptyset \}$, where $S$ is the training set, $B_d(x, \delta) \subset \mathbb{R}^d$ 
denotes the closed ball centered around $x \in \mathbb{R}^d$ of radius $\delta$. Then $N_{\delta}(S)$, as defined, is the collection of centres of each ball that intersects $\mathcal{W}$.

\textbf{H1}:
Let $\mathcal{Z}^{\infty} \coloneqq (\mathcal{Z} \times \mathcal{Z} \times \mathcal{Z} \times \ldots)$ denote the countable product endowed with the product topology and let $\mathfrak{B}$ be the Borel $\sigma$-algebra generated by $\mathcal{Z}^{\infty}$. Let $\mathfrak{F,G}$ be the sub-$\sigma$-algebras of $\mathfrak{B}$, generated by the collections of random variables given by $\{ L_S(w): w \in \mathcal{W}, n \geq 1 \}$ and $\Bigl\{ \mathbbm{1} \{ w \in N_{\delta}: \delta \in \mathbb{Q}_{>0}, w \in G, n \geq 1\} \Bigl\}$ respectively. There exists a constant $M \geq 1$, such that for any $A \in \mathfrak{F}$, $B \in \mathfrak{G}$, we have $\mathbb{P}[A \cap B] \leq M \mathbb{P}[A]\mathbb{P}[B]$.

This is known as the $\psi$-mixing condition \cite{bradley1983psi} and it is common in statistics for proving limit theorems. Smaller $M$ indicates that the dependence of $L_S(w)$ on the training sample $S$ is weaker.

\end{document}